\documentclass[11pt]{article}
\usepackage{natbib}
\usepackage{hyperref}       
\usepackage{url}    
\usepackage{bbm}
\usepackage{booktabs}       
\usepackage{amsfonts}       
\usepackage{nicefrac}       

\usepackage{ltablex}
\usepackage{multirow}
\usepackage{mathrsfs}  

\usepackage{amsmath, amssymb, amsthm, graphicx, bm}
\usepackage[ruled,vlined,linesnumbered,noend]{algorithm2e}

\usepackage{mwe}
\usepackage[section]{placeins}
\usepackage{enumitem}
\usepackage{nicefrac} 
\usepackage{mathtools}
\usepackage{cleveref}
\usepackage{thmtools,thm-restate}
\usepackage{dsfont}
\usepackage{tikz}
\usetikzlibrary{decorations.pathreplacing,calc}
\newcommand{\tikzmark}[1]{\tikz[overlay,remember picture] \node (#1) {};}

\newcommand*{\AddNote}[5]{%
    \begin{tikzpicture}[overlay, remember picture]
        \draw [decoration={brace,amplitude=0.5em},decorate,ultra thick,#5]
            ($(#3)!(#1.north)!($(#3)-(0,1)$)$) --  
            ($(#3)!(#2.south)!($(#3)-(0,1)$)$)
                node [align=center, text width=2.5cm, pos=0.5, anchor=west] {#4};
    \end{tikzpicture}
}

\usepackage{xcolor, color, colortbl}
\usepackage{tcolorbox}

\newcommand{\ind}{\ensuremath{\mathds{1}}}

\newtheoremstyle{definition}
  {}
  {}
  {\itshape}
  {}
  {\bfseries}
  {.}
  { }
  {\thmname{#1}\thmnumber{ #2}\thmnote{ (#3)}}

  \newtheoremstyle{theorem}
  {}
  {}
  {\itshape}
  {}
  {\bfseries}
  {.}
  { }
  {\thmname{#1}\thmnumber{ #2}\thmnote{ (#3)}}

\theoremstyle{theorem}

\newtheorem{lemma}{Lemma}[section]
\newtheorem{corollary}{Corollary}[section]

\theoremstyle{definition}
\newtheorem{definition}{Definition}[section]

\newtheorem{assumption}{Assumption}[section]

\newtheorem{remark}{Remark}[section]

\newcommand{\BigAbs}[1]{\Bigl\lvert#1\Bigr\rvert}

\DeclarePairedDelimiter\ceil{\lceil}{\rceil}
\DeclarePairedDelimiter\floor{\lfloor}{\rfloor}
\newcommand{\Qcal}{\mathcal{Q}}
\newcommand{\Acal}{\mathcal{A}}
\newcommand{\Ecal}{\mathcal{E}}
\newcommand{\mem}{\mathcal{M}}
\newcommand{\R}{\mathbb{R}}
\newcommand{\N}{\mathbb{N}}

\newcommand{\mistakes}{\textrm{MinorMistakes}}

\newcommand{\say}[1]{``#1''}

\newcommand{\error}{\textrm{OPT}}

\newcommand{\aexp}{\ensuremath{\textrm{ActiveExperts}}}
\newcommand{\bexp}{\textrm{BadExperts}}
\newcommand{\fexp}{\textrm{FailedExperts}}
\newcommand{\sexp}{\textrm{SaveExperts}}
\newcommand{\Fcal}{\mathcal{F}}
\newcommand{\om}{\max\limits_M}

\newcommand{\Wtilde}[1]{\stackrel{\sim}{\smash{#1}\rule{0pt}{1.3ex}}}
\usepackage{fullpage}
\usepackage{algorithmic}
\usepackage[margin=1in]{geometry}
\usepackage{environ}
\usepackage{mdframed}
\title{Learning what to remember}
\author{%
 Robi Bhattacharjee\\ University of California, San Diego\\ \texttt{rcbhatta@eng.ucsd.edu}
 \and
 Gaurav Mahajan \\ University of California, San Diego\\ \texttt{gmahajan@eng.ucsd.edu}
}

\begin{document}

\maketitle

\begin{abstract}
    We consider a lifelong learning scenario in which a learner faces a neverending and arbitrary stream of facts and has to decide which ones to retain in its limited memory. We introduce a mathematical model based on the online learning framework, in which the learner measures itself against a collection of experts that are also memory-constrained and that reflect different policies for what to remember. Interspersed with the stream of facts are occasional questions, and on each of these the learner incurs a loss if it has not remembered the corresponding fact. Its goal is to do almost as well as the best expert in hindsight, while using roughly the same amount of memory. We identify difficulties with using the multiplicative weights update algorithm in this memory-constrained scenario, and design an alternative scheme whose regret guarantees are close to the best possible.
    
\end{abstract}
\tableofcontents
\newpage
\section{Introduction}

A lifelong learning agent—a child learning language, or a robot exploring an environment, or a software program gathering information from the web—experiences a neverending stream of sensory or factual input. It cannot possibly retain all of this information and therefore has to decide what is important to remember. Ultimately, what should be remembered is information that will later be needed; thus the right memory policy might not be clear at the outset but will gradually be revealed by experience. 

The problem of \emph{learning what to remember} is a central challenge for lifelong learning systems with bounded memory \citep{thrun1995lifelong,mitchell2018never}. In this paper, we present a simple mathematical formalism based on online learning \citep{cesa-bianchi_lugosi_2006} that captures core aspects of this problem. In our model, at each time step, the learner receives either:
\begin{itemize}
\item a \emph{fact}, which can be thought of as a (question, answer) pair, like (“What is the capital of France?”, “Paris”), or 
\item a \emph{question} (“What is the capital of France?”). 
\end{itemize}
In the first case, the learner must decide whether or not to store the fact. It is constrained by having only enough memory for $M$ pieces of information. In the second case, the learner incurs a loss if it has not stored the corresponding fact. 

The stream of facts and questions is arbitrary and neverending. In choosing a policy for what to remember, the learner has access to $N$ \emph{experts}. These are subject to the same memory bound and could reflect different priorities: for instance, one expert might favor geographical facts, while another might select financial information. The learner’s goal is to do almost as well as the best expert in hindsight. That is, at any given time $t$, the learner should not have incurred too many more errors than the best expert at that time.

We think of an expert as an agent with enough memory for $M$ facts. As each new fact arrives, the agent has the option to store it, but in doing so might need to jettison some previously stored fact.

\paragraph{Some difficulties:}
A central idea in online learning is to act according to a \emph{weighted majority} of the experts \citep{mwu2012} and to continually adjust these weights as the experts accumulate different losses. Each of the $N$ experts, say expert $e$, is given a weight $w(e)$, and this weight is updated at each time step according to a rule of the form: “If expert $e$ makes a mistake, then reduce $w(e)$ (in some manner).”  

This general methodology immediately runs into two basic difficulties in our model. 
\begin{enumerate}
\item There is no easy way for the learner to tell which experts incur a loss at the current time step, because it does not know which facts each expert has in memory. The most obvious way to know this would be to keep a running simulation of the memories of all the experts, but this would require storing $MN$ facts.
\item Even if the learner somehow knows which experts incur a loss at each step, there is no easy way of maintaining the \emph{memory} of the weighted majority of experts: the set of facts that are retained by the majority of the experts. The overall number of such facts can be shown to be at most $2M$, which is not bad, but the problem is that as the weights shift, the composition of this majority-memory also shifts. Tracking these shifts in memory might require suddenly storing facts that appeared in the past, which is not possible in our framework (\Cref{rem:second} has an example).
\end{enumerate}

\subsection{Our Contributions}
Our starting point towards dealing with these two difficulties is to postpone the first by assuming that the learner has access to an \emph{oracle} that can tell it, for any expert and any fact, whether that expert currently has that fact in memory. We define this oracle formally in \Cref{or:unres}. We will later do away with this requirement.

The second problem, about tracking the majority-memory, remains. To cope with it, we introduce a different online learning scheme that changes weights \emph{very} infrequently, and in fact only uses two weights, 0 and 1. It uses $2M$ memory and has the following worst-case guarantee: If the loss of the best expert at time $t$ is $\error$, the loss of the algorithm at that time is $O(\error \log N + M \log N)$.

\newtheorem*{thm:upperboundunres}{(Informal) Theorem \ref{thm:upperboundunres}}
\begin{thm:upperboundunres}
Let $\Ecal$ be a set of $N$ experts with $M$ memory and $\error$ be the number of mistakes made by the best expert in $\Ecal$ by time $t$. Then there exists an algorithm with access to above mentioned oracle which by time $t$ makes at most $O(\error~ \log N + M \log N)$ mistakes using at most $2M$ memory. 
\end{thm:upperboundunres}

 With the second difficulty solved, we return to the first one and remove the need for an oracle. We consider experts of a particular form that we call  \emph{value-based}. Such an expert is fully specified by a value function $v: \Qcal \to \N$ that assigns a score to any given fact. The expert’s memory always consists of the $M$ highest-valued facts it has seen so far; when a new fact $(q,a)$ arrives, it decides whether to store this information by simply comparing $v(q)$ to the lowest-valued fact in its memory.

We show that if experts are of this type and the learner is only evaluated on questions taught before (\Cref{assump:seq}), then the learner can very coarsely keep track of the contents of all $N$ of their memories while using just $O(M)$ memory 
(which would otherwise naively require storing $MN$ questions or their corresponding values), and that this can be combined with the modified online learning algorithm introduced earlier to give similar mistake bounds.

\newtheorem*{thm:upperboundres}{(Informal) Theorem \ref{thm:upperboundres}}
\begin{thm:upperboundres}
Let $\Ecal$ be a set of $N$ value-based experts with $M$ memory and $\error$ denote the number of mistakes made by the best expert in $\Ecal$ by time $t$. Assume the learner is only evaluated on questions taught before. Then, there exists an algorithm which by time $t$ makes at most $O(\error~ \log N + M \log N)$ mistakes using at most $4M$ memory.
\end{thm:upperboundres}

Finally, we demonstrate that an additive term of $O(M \log N)$ in the regret is inevitable in this setting. Our lower bound applies also if the experts are value-based and the adversary satisfies \Cref{assump:seq}.

\newtheorem*{thm:lowerbound}{(Informal) Theorem \ref{thm:lowerbound}}
\begin{thm:lowerbound}
There exists a set $\Ecal$ of $N$ value-based experts using $M$ memory with the best expert making at most $\error$ mistakes such that any algorithm using $O(M)$ memory, with access to above mentioned oracle and only evaluated on questions taught before, makes at least $\Omega(\error + M\log N)$ mistakes.
\end{thm:lowerbound}
    Note that this lower bound in conjunction with our upper bounds show that for algorithms using $O(M)$ memory and with access to experts where the best expert makes no mistake, the mistake bound of $\Theta(M \log N)$ is actually tight. 

\subsection{Related Work}
Our model is a natural extension of online learning \citep{cesa-bianchi_lugosi_2006}. Closest to our setting is \cite{fu11} which considers restricting the memory used by an algorithm to store the weights. To the best of our knowledge, none of the previous work consider the setting where the experts store information. In contrast, our setting is concerned with experts which store \say{facts} and in turn how many facts (which is different than memory used to store weights considered in previous works) does a algorithm need to store for a reasonable regret.  

Previous works \cite{jacob16,alon2020} consider PAC learning (or SQ learning) with memory constraint where the adversary is restricted to sample from an unknown distribution and a fixed class of hypothesis. There is also a line of work \cite{ran16,garg18,shamir19} which prove lower bounds under memory constraints. Our work in comparison is in the more general online learning setting where the adversary is not restricted to a fixed distribution or hypothesis class.

Another line of work considers a non-adversarial sequence prediction setting where the sequence of facts is generated by a underlying model. \citep{hsu2012spectral,anandkumar2012method} uses spectral and tensor methods to learn the distribution generated by Hidden Markov Models by basically learning the parameters for the underlying model. Another strategy proposed in \cite{sharan2018prediction} is to just remember the last few facts. In contrast, in our framework, we consider the more general setting where the sequences can be adversarially chosen and the above strategies do not work well.

Our setting can be abstractly viewed as a combination of sketching and online learning. In sketching algorithms \cite{morris,khanna,countmin}, the goal is to compress data to approximately evaluate functions on it using small amount of memory. However naive use of sketching algorithms would lead to suboptimal memory use (quite often a dependence on time $T$).

Many learning models with explicit notions of memory are used in practice, e.g. recurrent neural networks \citep{Elman90findingstructure}, long short-term memory networks \citep{hochreiter1997long}, neural Turing machines \citep{graves2014neural}, memory networks \citep{weston2015memory}, and others. Our work is motivated by these memory-based architectures which basically use their long term memory as a \emph{dynamic} knowledge base and interact with (access or forget) this knowledge selectively using different mechanisms (including forgetting \citep{forget2021} and attention \citep{bahdanau2016neural}). 



\section{Setting}
\label{sec:framework}
Consider a set $\Qcal$ of all questions and set $\Acal$ of all answers. Let $\Phi: \Qcal \to \Acal$ be an arbitrary function which maps questions to answers. Throughout the paper, we will work with the corresponding set of facts $\Fcal=\{(q, \Phi(q)): q \in \Qcal\}$. 

We now introduce our learning framework: \emph{Online Question-Answering with Expert Advice}. Our framework can be thought of as a game between a learner and an adversary. The learner is allowed access to a \emph{memory} $\mem$ to store facts. It also receives advice from a group of $N$ experts $\Ecal$, each of which has access to its own memory $\mem_e$ of size $M$. 

\begin{figure}[h!]
    \begin{tcolorbox}
        At each time $t = 1,2,\ldots$
        \begin{enumerate}
            \item An adversary chooses either to teach or evaluate.
            \item If adversary chooses to teach, \begin{enumerate}
                \item the adversary shows a question answer pair $(q^{(t)}, \Phi(q^{(t)}))$ to the learner.
            \end{enumerate}
            \item Otherwise, if adversary choose to evaluate,
            \begin{enumerate}
                \item the adversary picks a question $q^{(t)}$.
                \item each expert $e$ incurs a cost $c_e^{(t)} = \ind\{ (q^{(t)}, \Phi(q^{(t)})) \notin \mem_e\}$ and the learner incurs a cost $c^{(t)} = \ind\{ (q^{(t)}, \Phi(q^{(t)})) \notin \mem\}$.
            \end{enumerate}
            \item Each expert $e$ (and learner) updates its memory $\mem_e$ (and memory $\mem$) by either choosing to store $(q^{(t)}, \Phi(q^{(t)}))$ or ignoring this information. They can also choose to remove any question answer pairs already stored in $\mem_e$.
        \end{enumerate}
        \end{tcolorbox}
        \caption{Online Question-Answering with Expert Advice}
        \label{fig:model}
\end{figure}

We now explain the framework (\Cref{fig:model}) at a high level. At each time step $t$, the adversary chooses either to teach or evaluate. If the adversary choose to \emph{teach}, it shows a fact $(q^{(t)}, \Phi(q^{(t)}))$ to the learner and all experts. On the other hand, if the adversary chooses to \emph{evaluate}, it chooses a question $q^{(t)}$. If the corresponding fact $(q^{(t)}, \Phi(q^{(t)}))$ is not stored in its memory $\mem$ (expert memory $\mem_e$), the learner (expert $e$) incurs a cost of $+1$. In both the scenarios, at this point, the learner (and any expert $e$) can either choose to store it in the memory $\mem$ ( and expert memory $\mem_e$) or ignore this information.  Note that both learner and experts can also choose to \emph{not} persist /remove information from its memory so it can store more important information later on.


Our goal is to find a decision-making algorithm which uses \emph{reasonable} amount of memory and makes least possible mistakes. We now formally define memory and mistakes for an algorithm.
\begin{definition}[Memory and Mistakes]
    We say a decision-making algorithm uses $B$ memory and makes $E$ mistakes by time $T$ if \[
        |\mem| \leq B ~\text{for all time}~ t \in [T] \quad \text{and}\quad \sum_{t =1}^T c^{(t)} = E
    \]
\end{definition}
\begin{remark}[Auxiliary state]
    \label{rem:aux}
    Note that we primarily care about the number of facts stored by the algorithm and we allow algorithms to use auxiliary state to store some state i.e. weights for different experts etc (for example, multiplicative weights update algorithm uses this auxiliary state to store the number of errors made by each expert). We do not allow storing facts in auxiliary state and all the algorithms in this work use at most $O(N)$ auxillary state.
\end{remark}

One approach for the learner could be to store all the $M N$ facts stored by the $N$ experts at any time $t$. Even though this approach will make at most the number of mistakes made by the best expert, it uses unreasonable amount of memory. 

On the other side of the spectrum, another approach for the learner could be to use multiplicative weights update algorithm \citep{mwu2012}. This approach uses only $2M$ memory! But as discussed in the introduction, there are two main difficulties with designing weighted majority style algorithms in this framework. In the next section, we will use an oracle to circumvent the \emph{first} problem of keeping track of experts memory and solve the \emph{second} problem of tracking the majority-memory.

\section{Upper Bound: Unrestricted Expert Access Oracle}
\label{sec:unres}
Recall that the first difficulty with weighted majority style algorithms is that its unclear how a learner can tell which facts are currently stored in the expert's memory without using $MN$ memory. In this section, we will assume access to an oracle, \emph{Unrestricted Expert Access Oracle}, which allows the learner to check for \emph{any} question if the corresponding fact is stored in an expert's memory.
\begin{definition}[Unrestricted Expert Access Oracle]
    \label{or:unres}
    At any time $t$, on inputting any expert $e$ and question $q$, Unrestricted Expert Access Oracle, denoted by $\mathcal{O}_u(e,q)$ returns True if the corresponding fact $(q,\Phi(q))$ is stored in the experts $e$'s memory $\mem_e$ and False otherwise.
\end{definition}
This oracle allows us to ignore the first issue for now (which we will solve later in \Cref{sec:res}). We still need to figure out the second difficulty: how to track the ever changing majority-memory. To illustrate this issue, we consider multiplicative weights update algorithm (\Cref{alg:mwu}) with Unrestricted Expert Access Oracle $\mathcal{O}_u$  and decompose it into two fundamental steps. In the \emph{weight update} step, the algorithm updates the weight assigned to each expert $w_e$ based upon the number of mistakes made by the expert $E_e$ (setting $w_e = (1-\gamma)^{E_e}$ for each expert). Then, in the \emph{memory update} step, it memorizes a fact if and only if the \emph{weighted majority} of experts memories the fact.

\begin{algorithm}[h!]
\SetArgSty{textrm}
    \caption{Multiplicative Weights Update algorithm}
    \label{alg:mwu}
    \SetAlgoLined
    \textbf{Initialize:} $E_e = 0$ for all experts $e \in \Ecal$\\
    \For{time $t = 1,2, \ldots$}{
        \uIf{the adversary chooses to evaluate\tikzmark{topw}}{
            \For{all expert $e\in \Ecal$}{
            Set $E_e = E_e + \ind\{ (q^{(t)}, \Phi(q^{(t)})) \notin \mem_e\}$ \qquad \qquad \qquad\tikzmark{rightw}\\
        Set $w_e = (1-\gamma)^{E_e}$.\tikzmark{bottomw} 
            }
         }
        Experts update their memory $\mem_e$ according to their rules.\\
        Set $\mem = \{(q^{(t)},\Phi(q^{(t)}))\} \cup \mem$ \tikzmark{topm}\\
        \For{fact $(q,a)$ in $\mem$}{
        Let $\sexp = \{ e: \mathcal{O}_u(e,q) = \text{True}\}$\\
        \If{$\sum_{e\in \sexp} w_e < \frac{1}{2} \sum_{e\in \Ecal} w_e$}{
            Remove $\{(q,a)\}$ from $\mem$.
            \tikzmark{bottomm}
        \AddNote{topw}{bottomw}{rightw}{~Weight Update}{red}
        \AddNote{topm}{bottomm}{rightw}{~Memory Update}{red}
            }
        }
        }        
\end{algorithm}

In essence, it constructs an implicit expert, the ``majority" expert, that always contains in memory precisely the pairs $(q, \Phi(q))$ that at least half of the weighted experts know. \Cref{alg:mwu}'s goal is to then maintain memory that approximates the memory of this implicit expert. 

\begin{remark}[Difficulty in maintaining majority-memory]
    \label{rem:second}
    The difficulty with this approach is that changes in the weights $w(e)$ can result in a dramatic change in the majority expert. For example, consider the simple case in which there are only two experts $e_1$ and $e_2$ (this is easily extended to $N$ experts) which hold completely different memory. If the majority expert shifts from $e_1$ to $e_2$ (say because $e_1$ made $O(1)$ more mistakes than $e_2$), this will result in a completely different majority-memory. Although we have the ability to immediately remove all unnecessary facts (corresponding to $e_1$'s memory) from memory, we do \emph{not} have the ability to add all necessary facts (facts corresponding to $e_2$'s memory) to memory; we may only memorize a fact when it is presented by the adversary (which causes $M$ mistakes in our example to catch up every time majority changes).
\end{remark}

This difficulty is compounded by the fact that our majority expert can potentially frequently change, as the weights are constantly updated at every time step. Because of this, it is plausible that our algorithm will always be ``behind" the majority expert by a significant lag, and this will cause problems for the standard analysis of the algorithm. 

To resolve this issue, we propose a modified algorithm which we call the Lazy Weights algorithm (\Cref{alg:mwu-la}). The main idea is to modify the \emph{weight update} step to update the weights less often to avoid the problem of a frequently changing majority expert. We do not make any changes to the \emph{memory update} step. Given access to Unrestricted Expert Access Oracle $\mathcal{O}_u$, we show that \Cref{alg:mwu-la} needs to store only $O(M)$ facts to achieve close to best possible mistake bound.
\begin{restatable}[Upper Bound: Online learning with infrequent weight changes]{thm}{upperboundunres}
    \label{thm:upperboundunres}
    Let $\Ecal$ be a set of $N$ experts with $M$ memory. Let $\error$ be the number of mistakes of the best expert in $\Ecal$ by time $t$. Then, by time $t$, \Cref{alg:mwu-la} with access to Unrestricted Expert Access Oracle $\mathcal{O}_u$ makes at most $6 \error \ceil{\log N} + 6M\ceil{\log N}$ mistakes using at most $2M$ memory and $O(N)$ auxiliary state.
\end{restatable}
We present a proof in \Cref{subsec:proofunres}. Next we provide a detailed description of \Cref{alg:mwu-la}.

\begin{algorithm}[h!]
    \SetArgSty{textrm}
        \caption{Lazy Weights Update algorithm}
        \label{alg:mwu-la}
        \SetAlgoLined
        Set $E_e = 0$ for all experts $e \in \Ecal$\\
        Set $\aexp = \Ecal$\\
        \For{time $t = 1,2, \ldots$}{
            \uIf{the adversary chooses to evaluate\tikzmark{topw}}{
                Set $E_e = E_e + \ind\{ \mathcal{O}_u(e, q) = \textrm{False}\}$ for all expert $e \in \Ecal$.\label{ln:update} \qquad \quad \qquad \tikzmark{rightw}\\
            Set $\bexp = \{e \in \aexp : E_e \geq M\}$.\label{ln:bad} \\
            \uIf{$|\aexp| \leq 3|\bexp|$}{
                    Set $\aexp = \aexp \setminus \bexp$ \label{ln:aexpupd}
                }
                \uIf{$\aexp = \emptyset$}{
                    Set $E_e = 0$ for all experts $e \in \Ecal$.\\
                    Set $\aexp = \Ecal$\label{ln:hardreset}\tikzmark{bottomw} }
             }
            Experts update their memory $\mem_e$ according to their rules.\\
            Set $\mem = \{(q^{(t)},\Phi(q^{(t)}))\} \cup \mem$ \tikzmark{topm}\\
            \For{fact $(q,a)$ in $\mem$}{
            Let $\sexp = \{ e: \mathcal{O}_u(e,q) = \text{True}\}$\\
            \If{$2|\sexp|< |\aexp|$}{
                Remove $\{(q,a)\}$ from $\mem$.\tikzmark{bottomm}
                \AddNote{topw}{bottomw}{rightw}{\quad Weight Update}{red}
            \AddNote{topm}{bottomm}{rightw}{~Memory Update}{red}
                }
            }
        }
    \end{algorithm}

\subsection{Lazy Weights Update algorithm}
\label{subsec:mwu-la}
\Cref{alg:mwu-la} follows the overall structure of multiplicative weights update algorithm (\Cref{alg:mwu}): it only modifies the \emph{weight update} step while keeping the rest of the algorithm intact.

It maintains a set of $\aexp$, which correspond to a ``weight" of 1. All other experts are given a weight of $0$. Initially, it gives ``{weight}'' of $1$ to each expert i.e. each expert is in the set $\aexp$. At each time $t$, it updates the number of mistakes made by each expert denoted by $E_e$ (it can track mistakes made by each expert using Unrestricted Expert Access Oracle $\mathcal{O}_u$). Next, it defines the set of candidate experts likely to be removed/de-weighted: \begin{align*}
    \bexp  &=\{e \in \aexp: E_e \geq M\}
\end{align*}
If $3|\bexp| \geq |\aexp|$, then it removes $\bexp$ from $\aexp$. However, if $\aexp$ ends up being empty, it resets $E_e = 0$ for all experts $e \in \Ecal$ and $\aexp = \Ecal$. Finally, it memorizes a fact if and only if the weighted majority of $\aexp$ memorize the fact (just like the multiplicative weights update algorithm).

\section{Upper Bound: Value-based Experts}
\label{sec:res}
In this section, we build upon the algorithmic ideas from \Cref{sec:unres} to show how to cope with the first difficulty: \emph{how a learner can tell which facts are currently stored in the expert's memory without using $MN$ memory?}, discussed in the introduction. For this, we consider a particular class of experts, which we call value-based experts.  
\begin{definition}[Value-based Expert]\label{defn:value_based}
    We say an expert $e$ is value-based with $M$ memory if there exists an injective function $v_e: \Qcal \to \N$ such that when shown a sequence of facts $(q_1, \Phi(q_1)), \ldots$ $(q_n, \Phi(q_n))$, it stores the $M$ facts with the largest value $v_e(q_i)$ i.e. \[
        (q_i, \Phi(q_i)) \in \mem_e \iff v_e(q_i) \geq \om(\{v_e(q_1), v_e(q_2), \ldots, v_e(q_n)\}) 
    \] where $\om$ represents the $M^{\text{th}}$ largest element in a set. We will also refer to this $M$th largest element as the threshold of expert $e$, denoted $T(e)$. 
\end{definition}
An example for such an expert is one which priorities geographical facts to be remembered and therefore sets $v_e$ for such facts to be large. Note however that \emph{temporal} experts of the form \say{store last $M$ facts} can not be represented as a value-based experts.

How does assuming experts to be value-based help with the first difficulty? Recall that for each value-based expert $e$ and value function $v_e$, there exists a threshold $T_e^*$ such that $$(q, \Phi(q)) \in \mathcal{M}_e \iff v_e(q) \geq T_e^*.$$ While we do not have access to the thresholds $T_e^*$ (which change over time), we \textit{do} have access to the value function $v_e$. Thus, if we can \textit{estimate} $T_e^*$ for all experts $e \in \mathcal{E}$. We will denote our estimate as $T_e$ which we will then use to simulate the oracle $\mathcal{O}_u$, which in turn allows us to apply the same methods we used in \Cref{alg:mwu-la}. However, the naive way of maintaining thresholds by storing the $M$ largest $v_e$ values for each expert $e$ requires $MN$ auxiliary state. We will show in \Cref{alg:mwures}, how to use $O(M)$ memory (to store questions) and $O(N)$ auxiliary state (to store thresholds) to maintain approximate lower bounds for the thresholds.

In addition to our assumption about value-based experts, we will need another assumption regarding our adversary, which we will call the \textit{sequential adversary} assumption.

\begin{assumption}[Sequential Adversary]
	\label{assump:seq}
We assume that the adversary only evaluates on question $q$ if the fact $(q,\Phi(q))$ has been previously taught. In particular, if the adversary chooses to evaluate on $q^{(t)}$  then there exists $s < t$ with $q^{(s)} = q^{(t)}$ such that time $s$ was a teaching instance. 
\end{assumption}
We note that this assumption is pretty natural (for example, most exams only evaluate on facts taught before). In fact, it only restricts the adversary from evaluating on facts never shown before, in which case all experts and the learner anyways always make a mistake.

In summary, we make the following two assumptions: (1) all experts $e \in \mathcal{E}$ are value based experts, and we have access to the value functions $v_e$ for all $e$, (2) the adversary is sequential (\Cref{assump:seq}). Under these assumptions, we will show that \Cref{alg:mwures}, a modification of \Cref{alg:mwu-la}, uses $O(M)$ memory, $O(N)$ auxillary state and makes same number of mistakes.

\begin{restatable}[Upper Bound: Value-based Experts]{thm}{upperboundres}
    \label{thm:upperboundres}
    Let $\Ecal$ be a set of $N$ value based experts with $M$ memory and the adversary satisfies \Cref{assump:seq}. Let $\error$ denote the number of mistakes made by the best expert in $\Ecal$ by time $t$. Then, by time $t$, \Cref{alg:mwures} makes at most $6\error \ceil{\log N} + 6M \ceil{\log N}$ mistakes using at most $4M$ memory and $O(N)$ auxillary state.
    \end{restatable}
We provide a complete proof in \Cref{subsec:proof-ideas-res}. Next, we give a detailed description of \Cref{alg:mwures} in \Cref{subsec:res}. 

\subsection{Value Based Lazy Weights Update algorithm}
\label{subsec:res}

\Cref{alg:mwures} uses the same core ideas as \Cref{alg:mwu-la}: we maintain a set of $\aexp$, which correspond to a ``weight" of 1. All other experts are given a weight of $0$. These sets are maintained in the same lazy fashion: we remove experts from $\aexp$ only when a  significant portion of $\aexp$ have made at least $M$ mistakes. 

The main difference in this setting that we do not have access to Unrestricted Expert Access Oracle $\mathcal{O}_u$, and consequently need a way to determine when an expert makes a mistake. We do this by  maintaining estimates $T_e$ of the true thresholds $T_e^*$. These threshold, $T_e$ are then used to simulate the Unrestricted Expert Access Oracle $\mathcal{O}_u$ by checking if $v_e(q) \geq T_e$ in \Cref{ln:orc}. 

One significant challenge with this strategy is that $T_e$ can significantly underestimate the threshold. To account for this, we will also require an \emph{additional} estimated threshold, $T_e^{pre}$ which estimates the value of $T_e^*$, at the latest time before the current time during which the set $\aexp$ was changed.

We first describe the subroutine--\Cref{alg:update_pre_threshold} for maintaining $T_e^{pre}$. 
Here, we store a set, $\textrm{MinorMistakes}$, which maintains an estimate of the set of \say{minor mistakes} , i.e. questions where (1) the algorithm makes a mistake and (2) strictly less than half of $\aexp$ make a mistake (according to our estimated thresholds $T_e$). We then set $T_e^{pre}$ to be $\max_M\{v_e(q): q \in \textrm{MinorMistakes}\}$ and correspondingly update the mistakes count for experts.
\begin{algorithm}[h!]
\SetArgSty{textrm}
    \caption{$\textrm{UpdatePreThreshold}(q^{(t)})$}
    \label{alg:update_pre_threshold}
    \SetAlgoLined
    $\mistakes = \mistakes \cup q^{(t)}$\label{ln:pre1}\\
    \For{$e$ in $\aexp$}{
    	$x \leftarrow \om\{v_e(q): q \in \mistakes\}$. \label{ln:upup1}\\
    	$T^{pre}_e = \max(x, T^{pre}_e)$.\label{ln:upup2} 
    }
    \For{$q$ in $\mistakes$}{
    	\uIf{$v_e(q) < T^{pre}_e$ for at least half of all $e \in \aexp$}{
    		Set $E_e = E_e + 1$ for all $e$ with $v_e(q) < T^{pre}_e$. \label{ln:u2error}\\
    		Remove $q$ from $\mistakes$.\label{ln:uprem} \\
    	}
    }
\end{algorithm}

Next, we describe the subroutine (\Cref{alg:update_threshold}) for maintaining $T_e$. In \Cref{alg:update_threshold} we set $T_e$ as the $M$th largest value of $v_e(q)$ for questions $q$ in $\mathcal{M} \cup \textrm{MinorMistakes}$.

\begin{algorithm}[h!]
\SetArgSty{textrm}
    \caption{$\textrm{UpdateThreshold}$}
    \label{alg:update_threshold}
    \SetAlgoLined
    \If{$|\mem \cup \mistakes| \geq M$}{
			\For{$e \in \aexp$}{
				$x \leftarrow \om\{v_e(q): q \in \mem \cup \mistakes\}$\label{ln:up1}\;
			$T_e \leftarrow \max(x, T_e)$\label{ln:up2}\;
			}
	}
\end{algorithm}

Other than these subroutines, \Cref{alg:mwures} is almost same as \Cref{alg:mwu-la}, except that the error counts $E_e$ are managed somewhat differently. \Cref{alg:mwu-la} immediately incremented the error counter $E_e$ by $1$ upon realizing that expert $e$ makes a mistake on question $q^{(t)}$. In \Cref{alg:mwures}, due to the inherent uncertainty in identifying mistakes we implement a more deferred strategy. In \Cref{ln:uerror}, we increment error counts when at least half of $\aexp$ make a mistake (according to our estimates $T_e$). Otherwise, we first move the question to $\mistakes$, and only upon removing the question from $\mistakes$ do we increment the counters (\Cref{ln:u2error} of \Cref{alg:update_pre_threshold}).

\begin{algorithm}[h!]
\SetArgSty{textrm}
    \caption{Value Based Lazy Weights Update algorithm}
    \label{alg:mwures}
    \SetAlgoLined
    Set $E_e, T^{pre}_e, T_e = 0$ for all experts $e \in \Ecal$\\
    Set $\aexp = \Ecal$\\
    Set $\textrm{MinorMistakes} = \emptyset$\\
    \For{time $t = 1,2, \ldots$}{
        \uIf{the adversary evaluated and $q^{(t)} \not \in \mem$\tikzmark{topw}}{
            Set $\fexp = \{e \in \aexp: v_e(q^{(t)}) < T_e\}$.\\
            \label{ln:serror}\uIf{$2|\fexp| < |\aexp|$}{
                \textbf{Run} ${\textrm{UpdatePreThreshold}}(q^{(t)})$ 
				}
			\uElse{
				Set $E_e = E_e + 1$ for all $e \in \fexp$ \label{ln:uerror}
			}
            
        Set $\bexp = \{e \in \aexp ~:~ E_e \geq M\}$ \label{ln:update2} \\
        \uIf{$|\aexp| \leq 3|\bexp|$}{
            	$\aexp = \aexp \setminus \bexp$\label{ln:update3}\\
            }
            \uIf{$\aexp = \emptyset$}{
            	$\aexp = \Ecal$, $E_e = 0$,  for all experts $e \in \Ecal$.\tikzmark{bottomw} }
         }
            \textbf{Run} ${\textrm{UpdateThreshold}}$ \\
            Set $\mem = \{(q^{(t)},\Phi(q^{(t)}))\} \cup \mem$\\
        \For{fact $(q,a)$ in $\mem$\label{ln:allmemory}\tikzmark{topm}}{
            Set $\sexp = \{ e \in \aexp: v_e(q) \geq  T_e\}$\label{ln:orc}\\
        \If{$2|\sexp| < |\aexp|$}{
            Remove $\{(q,a)\}$ from $\mem$ \label{ln:allmemoryend}.\tikzmark{bottomm}
            \AddNote{topw}{bottomw}{rightw}{~Weight Update}{red}
        \AddNote{topm}{bottomm}{rightw}{~Memory Update}{red}
            }
        }
    }
\end{algorithm}

\section{Lower Bound}
\label{sec:lowerbound}
We now discuss the minimum memory any algorithm needs to behave competitively with respect to the best expert. We show that for any algorithm with $O(M)$ memory and arbitrary amount of auxillary state, the additive $\Omega(M \log N)$ mistakes are inevitable. 

\begin{restatable}[Lower Bound]{thm}{lowerbound}
    \label{thm:lowerbound}
    Fix $c, N, M$ and $\error$ to be positive natural numbers. There exists a set $\Ecal$ of $N$ value based experts using $M$ memory with the best expert making at most $\error$ mistakes and adversary satisfying \Cref{assump:seq} such that any algorithm $A$ using $cM$ memory and with access to Unrestricted Expert Access Oracle $\mathcal{O}_u$, in the worst case, makes at least $\Omega(\error + M\log N)$ mistakes.
\end{restatable}

We provide a complete proof in \Cref{sec:lowproof}. On a high level, the proof basically repeats the following simple strategy (for $M$ memory algorithms): Divide the set of $N$ experts into two groups which remember two different set of $M$ facts. Irrespective of what algorithm chooses to remember, we can always choose a set of questions to evaluate it on, such that algorithm makes $\approx M/2$ mistakes and half of the experts make $0$ mistake. Repeating this $\log(N)$ times (recursively on the set of experts which make $0$ mistakes) gives the required bound.


\section{Proofs}
In this section, we provide the proofs for our upper bounds: \Cref{thm:upperboundunres}-\ref{thm:upperboundres} and lower bound: \Cref{thm:lowerbound}. Our proof for \Cref{thm:upperboundres} builds on the proof of \Cref{thm:upperboundunres}, so we first present its proof.
\subsection{Proof for Theorem \ref{thm:upperboundunres}}
\label{subsec:proofunres}
In this subsection, we will provide a proof and intuition for memory and mistake bound of \Cref{thm:upperboundunres}. We start with proving that \Cref{alg:mwu-la} uses at most $2M$ memory. To prove this, we will use the following helper lemma which helps us analyze the behavior of ``majority'' expert.
\begin{lemma}[Helper Lemma for Majority Expert]
    \label{lem:basemaj}
	Consider an arbitrary weighting function $w: \Ecal \to \{0,1\}$. Let $D$ be a set of facts and $b: \Ecal \times \Fcal \to \{0,1\}$ be a binary function such that for all experts $e \in \Ecal$ \begin{equation}
        \label{eq:bm}
        \sum_{f \in D} b(e,f) \leq M
    \end{equation} Then, the following is true \[
        \BigAbs{\Big\{f \in D: \sum_{e\in \Ecal} w(e) b(e,f) \geq \frac{1}{2} \sum_{e\in \Ecal} w(e)\Big\}} \leq 2M
    \]
\end{lemma}
\begin{proof}
    Let's denote the set above by $D'$ and suppose $|D'| = k$. Then, using \Cref{eq:bm}, we get \begin{align*}
        \frac{k}{2} \sum_{e\in \Ecal} w(e) &\leq \sum_{f \in D' } \sum_{e\in \Ecal} w(e) b(e,f)\\
         &= \sum_{e\in \Ecal}  w(e) \Big( \sum_{f\in D} b(e,f)\Big) \\
         &\leq M \sum_{e\in \Ecal} w(e)
    \end{align*}
where the first step follows from definition of $D'$, the second step follows from $D' \subset D$ and the last step follows from \Cref{eq:bm}.
\end{proof}
Invoking the above lemma for $b(e,f)$ defined as whether expert $e$ stored fact $f$ or not, and weighting $w(e)$ defined as whether expert $e$ is in $\aexp$ or not proves our memory bound. This argument also shows that multiplicative weights update algorithm uses at most $2M$ memory.
\begin{lemma}[Memory Bound]
    \label{lem:weightedmajority}
	For all time $t$, $|\mem| \leq 2M$.
\end{lemma}
\begin{proof}
    Let $b(e,(q',\Phi(q'))) = 1$ if and only if $\mathcal{O}_u(e,q') = True$ and $0$ otherwise. Then, because each expert stores at most $M$ facts from $\mem \cup \{(q,a)\}$, we get that for all experts $e \in \Ecal$ \[
        \sum_{f \in \mem \cup \{(q,a)\}} b(e,f) \leq M
    \] Define weighting $w : \Ecal \to \{0,1\}$ given by $w(e) = \ind\{e \in \aexp\}$. Then, by the last step in the algorithm at each time $t$, $(q,a)$ is \emph{not} removed from (persisted in) $\mem$ if and only if \[
        \sum_{e\in \Ecal} w^{(t)}(e) b(e,f) \geq \frac{1}{2} \sum_{e\in \Ecal} w^{(t)}(e)
    \] Therefore, by \Cref{lem:basemaj}, the claim follows.
\end{proof}

Now, we need to show that \Cref{alg:mwu-la} does not make too many mistakes compared to the best expert. We would like to distinguish between the mistakes made by \Cref{alg:mwu-la} based on if majority of the $\aexp$ also made a mistake or not. We define such mistakes as being minor or major mistakes.

\begin{definition}[Minor and Major Mistakes]
    \label{def:major-minor}
    We partition the mistakes made by the algorithm into:
    \begin{enumerate}
        \item A \textbf{minor mistake} is a question-time pair $(q,t)$ in which the algorithm makes a mistake and strictly less than half of $\aexp$ make a mistake. This can be thought of a question where algorithm make a mistake, but the implicit majority expert succeeds.
        \item A \textbf{major mistake} is a question-time pair $(q,t)$ in which the algorithm makes a mistake and at least half of $\aexp$ make a mistake. This can be thought of a question where both algorithm and the majority expert make mistakes.
    \end{enumerate}
\end{definition}

As we shall see later, number of major mistakes made by \Cref{alg:mwu-la} are much easier to control. Therefore, we first prove our main lemma for controlling minor mistakes made by \Cref{alg:mwu-la}. Here we show that \Cref{alg:mwu-la} can only make at most $2M$ minor mistakes between two consecutive $\aexp$ update using our helper lemma (\Cref{lem:basemaj}) for analysing majority expert.

\begin{lemma}[Minor Mistakes Between Updates]
\label{lem:minor}
Let time $t < t'$ be such that $\aexp$ were not updated i.e. \Cref{ln:aexpupd} was not executed between $t$ and $t'$. Then there are at most $2M$ minor mistakes between times $t$ and $t'$ (inclusive). 
\end{lemma}

\begin{proof}
We start by considering the state at time $t$. Let $S^{(t)}$ denote the subset of facts $D$ shown till time $t$ that at least half of all active experts know and the algorithm does \textit{not} have in memory. First, we have that $|S^{(t)}| \leq 2M$. 
To see why this is true, let $b(e,(q',\Phi(q'))) = 1$ if and only if $\mathcal{O}_u(e,q') = True \land e \in \aexp$; and $0$ otherwise. Then, because each expert stores at most $M$ facts from $D$, we get that for all experts $e \in \Ecal$ \[
    \sum_{f \in D} b(e,f) \leq M
\] Define weighting $w : \Ecal \to \{0,1\}$ given by $w(e) = \ind\{e \in \aexp\}$. Then, by definition of $S^{(t)}$, $(q,a) \in S^{(t)}$ if and only if \[
    \sum_{e\in \Ecal} w(e) b(e,f) \geq \frac{1}{2} \sum_{e\in \Ecal} w(e)
\] Therefore, by \Cref{lem:basemaj}, $|S^{(t)}| \leq 2M$.

Next, we look at what happens between time $t$ and $t'$. Observe that every minor fail reduces $S^{(t)}$ by exactly one. This is true, because for every minor mistake, (1) we see a fact $f$ from $S^{(t)}$ by definition and (2) by our algorithm, we memorize the fact when we make a mistake since $f$ is the memory of at least half of the active experts. This completes the proof. 
\end{proof}
Using the bound on minor mistakes between updates of $\aexp$, we can easily bound the number of total mistakes: minor and major, made by \Cref{alg:mwu-la} in comparison to best expert.
\begin{lemma}[Mistake Bound Between Updates]\label{lem:double_counting} 
	Let $L^{(t)}$ be the number of mistakes made by the algorithm by time $t$. Consider times $t< t'$ such that $L^{(t)} + 6M < L^{(t')}$. Then $\aexp$ were updated i.e. \Cref{ln:aexpupd} was executed between time $t$ and $t'$. 
\end{lemma} 

\begin{proof}
Assume towards a contradiction that $\aexp$ were \textbf{not} updated..
Let $A$ denote the number of $\aexp$ at time $t$, $R$ denote the number of major mistakes betwen $t$ and $t'$, and  $X$ denote the number of $\bexp$ at time $t'$ i.e. $\aexp$ with $E_e\geq M$ at time $t'$. Our strategy will be to double count the number of pairs $(e, s)$ where $t \leq s \leq t'$ is a \emph{major mistake}, and $e$ is an expert that got the question at time $s$ incorrect. 

We first upper bound this count. For each expert with $E_e < M$, it can clearly be part of at most $M$ pairs. Since, each expert is part of at most $R$ pairs, we get that there are most $XR + (A-X)M$ such pairs.

Next, we can easily lower bound the count. Each major mistake $(e,s)$ has at least $A/2$ experts that get the corresponding question wrong (by definition). Thus, there are at least $AR/2$ such pairs. 

Together, the upper and lower bound imply that $$XR + (A-X)M \geq AR/2.$$  Since, by \Cref{lem:minor}, we make at most $2M$ minor mistakes, and we must have at least $4M$ major mistakes (since we make at least $6M$ mistakes in total), i.e. $R \geq 4M$. Substituting this above implies $X \geq A/3$, which is in contradiction with algorithm description (\Crefrange{ln:update2}{ln:update3}) since then $\aexp$ will be updated.


\end{proof}

\Cref{lem:double_counting} basically means that whenever the algorithm makes $6M$ mistakes, $1/3$rd of the $\aexp$ make $M$ mistakes. Note that this can only happen at most $\log N$ times before every expert must have made at least $M$ mistakes. This immediately proves our main result--\Cref{thm:upperboundunres}.
\begin{proof}[Proof of \Cref{thm:upperboundunres}]
    We first define some notation. We say time $t$ is a \textbf{hard reset} if \Cref{ln:hardreset} of Algorithm \ref{alg:mwu-la} is executed at time $t$ i.e. all errors are set again to $0$, and $\aexp$ is set to $\Ecal$ at time $t$.

    We claim that: if $t < t'$ are 2 consecutive hard resets, then \[L^{(t)} + 6M \ceil{\log N} \geq L^{(t')}\quad \text{and}\quad \error^{(t')} \geq  \error^{(t)} + M\, .\] We now prove our theorem using this claim. Let $t$ be any time, and let $0 = t_0 < t_1 < \dots < t_r < t$ be all the hard resets smaller than $t$. Applying our first claim, we see that $\error$ grows by at least $M$ between every $t_i, t_{i+1}$, implying that \[\error^{(t)} \geq rM\, .\] 
    On the other hand, the number of mistakes made by the algorithm grows by at most $6M\ceil{\log  N}$ between every $t_i, t_{i+1}$. Thus \[L^{(t)} \leq 6rM\ceil{\log  N} + 6M \ceil{\log  N}\, .\]  Substituting our bound on $\error^{(t)}$, proves the theorem.

    We now prove the claim. Directly before a hard reset, by definition every expert satisfies $E_e \geq M$. This immediately implies that \[\error^{(t')} \geq  \error^{(t)} + M\, ,\] since $E_e$ was reset to $0$ at time $t'$ for all experts $e\in \Ecal$ (by definition of a hard reset) and only incremented by $1$ when expert $e$ makes a mistake. Next, by \Cref{lem:double_counting}, every $6M$ mistakes by the algorithm corresponds to removing at least $1/3$ of all $\aexp$. Thus executing this process at most $\ceil{\log N}$ times results in the hard reset, meaning that we can have at most $6M\ceil{\log N}$ mistakes incurred by the algorithm between $t, t'$ which proves our claim. 
\end{proof}

\subsection{Proof for Theorem \ref{thm:upperboundres}} \label{subsec:proof-ideas-res}
In this subsection, we provide a proof and intuition for memory and mistake bound of \Cref{alg:mwures}. The general proof structure of \Cref{thm:upperboundres} closely follows that of \Cref{thm:upperboundunres}. We will have the same main steps: bounding the total memory used by the algorithm (\Cref{lem:weightedmajority}), defining minor and major mistakes (\Cref{def:major-minor}), bounding the number of minor mistakes between updates (\Cref{lem:minor}), and bounding the number of total mistakes between updates (\Cref{lem:double_counting}).
	
The key difference is that \Cref{alg:mwures} does not get full information of when experts actually make mistakes; it instead has to use its maintained thresholds, $T_e, T_e^{pre}$ to estimate when this happens. This means that the error counts, $E_e$ are not true measures of the number of mistakes each expert makes. To account for this, we will first apply the arguments from \Cref{sec:res} to show that \Cref{alg:mwu-la} has a similar performance to \Cref{alg:mwures} \textit{with respect to its error counts}, $E_e$. We will then show that the error counts $E_e$ are indeed underestimates of the true error counts, $E_e^*$, which will then imply the theorem. 

\begin{lemma}[Memory Bound]\label{lem:mem_bound_res}
For all times $t$, $|\mem| \leq 2M$.
\end{lemma}

\begin{proof}
We closely adapt the proof of \Cref{lem:weightedmajority}. The only difference here is that we define the binary function $b$ by using the thresholds $T_e$ instead of applying the oracle $\mathcal{O}_u$. 

To that end,  let $b(e, q) = 1$ if $v_e(q) \geq T_e$ and $0$ otherwise. At the end of executing \textrm{UpdateThreshold}, by definition there exist at most $M$ questions $(q, a)	 \in \mem$ such that $v_e(q) \geq T_e$. It follows that for all $e \in \Ecal$, $$\sum_{(q, a) \in \mem} b(e, q) \leq M.$$

Define weighting $w : \Ecal \to \{0,1\}$ given by $w(e) = \ind\{e \in \aexp\}$. Then, by the last step in the algorithm at each time $t$, $(q,a)$ is \emph{not} removed from (persisted in) $\mem$ if and only if \[
        \sum_{e\in \Ecal} w^{(t)}(e) b(e,q) \geq \frac{1}{2} \sum_{e\in \Ecal} w^{(t)}(e)
    \] Therefore, by \Cref{lem:basemaj}, the claim follows.
\end{proof}

Next, we give an updated definition of major and minor mistakes.

\begin{definition}[Major and Minor Mistakes]\label{def:new_major_minor}
We classify instances in which the algorithm makes mistakes as follows.
\begin{enumerate}
	\item A \textbf{minor mistake} is a question $q$ that is stored inside $\mistakes$. 
	\item A \textbf{major mistake} is a question $q$ for which the error counters $E_e$ are incremented for at least half of all experts in $\aexp$. This can either occur in \Cref{ln:uerror} of \Cref{alg:mwures}, or \Cref{ln:u2error} of \Cref{alg:update_pre_threshold}.
\end{enumerate}
\end{definition}

The main differences in this definition are that they are no longer time specific (i.e. a question answer pair $(q, a)$ is not necessarily classified as a major or minor mistake at the time it is streamed), and that they are defined with respect to actions the algorithm takes (rather than an oracle). Furthermore, minor mistakes are mutable: it is possible for a question $q$ to be considered a minor mistakes at time $t$ but later be converted to a major mistake (\Cref{ln:u2error} of \Cref{alg:update_pre_threshold}). Thus, for any given time, we define the number of minor mistakes made between updates as the number of elements inside $\mistakes$, which we bound in the following lemma. Note that while the classification of $q$ can potentially change from minor mistake to major mistake, the key idea is that at all times, all mistakes have a precise classification as to whether they are major or minor (which is crucial in the proof of \Cref{lem:double_counting_res}).

\begin{lemma}[Minor Mistakes]\label{lem:minor_mistake_res}
$|\mistakes| \leq 2M$ at all times. 
\end{lemma}

\begin{proof}
The idea here closely follows the proof of \Cref{lem:mem_bound_res}. The only difference is that we define $b$ with respect $T_e^{pre}$, as $b(e, q) = 1$ if $v_e(q) \geq T_e^{pre}$ and $0$ otherwise. Since executing \text{UpdatePreThreshold} enforces that $\sum_{q \in mistakes} b(e, q) \leq M$ for all experts, the same argument follows.
\end{proof}

We now bound the total number of mistakes between updates.

\begin{lemma}[Mistake Bound Between Updates]\label{lem:double_counting_res} 
	Let $L^{(t)}$ be the number of mistakes made by the algorithm by time $t$. Consider times $t< t'$ such that $L^{(t)} + 6M < L^{(t')}$. Then $\aexp$ were updated i.e. \Cref{ln:aexpupd} was executed between time $t$ and $t'$. 
\end{lemma} 

\begin{proof}
This proof almost identically follows the proof of \Cref{lem:double_counting}. Assume towards a contradiction that $\aexp$ were \textbf{not} updated.
Let $A$ denote the number of $\aexp$ at time $t$, $R$ denote the number of major mistakes betwen $t$ and $t'$, and  $X$ denote the number of $\bexp$ at time $t'$ i.e. $\aexp$ with $E_e\geq M$ at time $t'$. Note that we are using \Cref{def:new_major_minor} for major and minor mistakes. 

Our strategy will be to double count the number of pairs $(e, q)$ where $q$ is a major mistake, and $e$ is an expert whose error counter is incremented on behalf of $q$ (i.e. in \Cref{ln:uerror} of \Cref{alg:mwures}, or \Cref{ln:u2error} of \Cref{alg:update_pre_threshold}.). Similar to \Cref{lem:double_counting}, this gives
$$XR + (A-X)M \geq AR/2.$$  

Given this equation, we finish the proof by applying the same reasoning as \Cref{lem:double_counting}. The only remaining argument is to prove that $R \geq 4M$. 

Let $q$ be an arbitrary question that our algorithm makes a mistake on. If at least half of all active experts also make a mistake (based on their estimated thresholds $T_e$), then $q$ is a major mistake (\Cref{ln:uerror} of \Cref{alg:mwures}) and is \textit{never} considered a minor mistake. If this does not occur, then $q$ is appended to $\mistakes$ (\Cref{ln:pre1} of \Cref{alg:update_pre_threshold}) and is consequently considered a minor mistake during \textit{all times it remains in $\mistakes$}. Finally, if $q$ is removed from $\mistakes$, then it is necessarily considered a major mistake again (\Cref{ln:u2error} of \Cref{alg:update_pre_threshold}). In summary, while the classification of $q$ can potentially change from minor mistake to major mistake, the key idea is that at all times, all mistakes have a precise classification as to whether they are major or minor.

Finally, since \Cref{lem:minor_mistake_res} implies that at all times the number of minor mistakes is at most $2M$, we must have that $R \geq 4M$ as all other mistakes must be classified as major. 


\end{proof}

By a direct adaptation of the proof of Theorem \ref{thm:upperboundunres}, we have the following corollary.

\begin{corollary}
Define a perceived error of an expert $e$ to be any instance in which $E_e$ is incremented by $1$. Let $\Wtilde{\error}$ denote the smallest perceived error of any expert. Then Algorithm \ref{alg:mwures} makes at most $6\Wtilde{\error}\lceil \log n \rceil + 6M \lceil \log N \rceil$ mistakes using at most $4M$ memory and $O(N)$ auxiliary state. 
\end{corollary}

\begin{proof}
This directly follows the proof of \Cref{thm:upperboundunres}. We use the same definition of \textbf{hard reset}. Like proof of \Cref{thm:upperboundunres},  if $t < t'$ are 2 consecutive hard resets, then \[L^{(t)} + 6M \ceil{\log N} \geq L^{(t')}\quad \text{and}\quad \Wtilde{\error^{(t')}} \geq  \Wtilde{\error^{(t)}} + M\, .\] Let $t$ be any time, and let $0 = t_0 < t_1 < \dots < t_r < t$ be all the hard resets smaller than $t$. Like \Cref{thm:upperboundunres}, we get  \[\Wtilde{\error^{(t)}} \geq rM\, , \quad \text{and} \quad L^{(t)} \leq 6rM\ceil{\log  N} + 6M \ceil{\log  N}\, .\] This proves the mistake bound. Moreover, the only additional memory this algorithm uses are: the set $\mistakes$, and the thresholds $T_e, T_e^{pre}$. Since there are $O(N)$ thresholds and since $|\mistakes| \leq 2M$ (\Cref{lem:minor_mistake_res}), the memory bound follows. 
\end{proof}

Finally, to prove Theorem \ref{thm:upperboundres}, we need to show that $\Wtilde{\error}$ is an underestimate of the true number of mistakes made by the best expert, $\error$. 

\begin{proof}[Proof of \Cref{thm:upperboundres}]
As we stated above, it suffices to show that $\Wtilde{\error} \leq \error$. For any expert $e$, let $T_e$ and $T_e^{pre}$ be the values of thresholds at time $t$ while executing \Cref{ln:uerror} or \Cref{ln:u2error}.  Define $T_e^{*}$ to be the true threshold of expert $e$ at that time, and $T_e^{pre, *}$ to be the true threshold of expert $e$ at time $s$ where $s$ was the last time that the set of active experts was updated by our algorithm. Note that such true thresholds must exist because all our experts are assumed to be value based experts (\Cref{defn:value_based}). 

We claim that $T_e \leq T_e^{*}$, and $T_e^{pre} \leq T_e^{pre, *}$. These claims finishes the proof as it implies that every increment to error counter $E_e$ corresponds to a instance in which expert $e$ actually made a mistake.

The core idea for both of these claims is that $T^*_e$ or $T_e^{pre, *}$ is $M$th largest value of a certain set of questions and every time we update either of $T_e$ or $T_e^{pre, *}$, we are updating them to the $M$th largest value of its subset. 

First, we prove $T_e \leq T_e^{*}$. This follows since, $T_e$ is updated in \Cref{ln:upup1} of \Cref{alg:update_threshold} to be the $M$th largest value of a \emph{subset} of observed questions and $T_e^{*}$ is defined to be the $M$th largest value of \emph{all} observed questions. 

The proof for $T_e^{pre}$ is slightly more involved. $T_e^{pre, *}$ is defined to be the $M$th largest value of \emph{all} observed questions before time $s$ where $s$ was the last time that the set of active experts was updated by our algorithm. To prove this we need to show that the set used in \Cref{ln:upup1}, $\mistakes$ is a subset of \emph{all} observed questions before time $s$ or equivalently $\mistakes$ does not contain any question that was \emph{first} streamed after time $s$. 

If a question $q$ is added to $\mistakes$, then two things must happen. First, at least half of all experts got the question correct according to the current estimated thresholds $T_e$. Second, our algorithm must get it wrong. The key observation is that if the question $q$ were streamed for the first time \textit{after} $s$, then our algorithm would have memorized it. This is because all maintained thresholds are non-decreasing. Thus, since our algorithm got it wrong, $q$ must have been streamed \textit{before} $s$.

\end{proof}

\subsection{Proof for Theorem \ref{thm:lowerbound}}
\label{sec:lowproof}
In this subsection, we provide the proof for our lower bound--\Cref{thm:lowerbound}. On a high level, the proof basically repeats the following simple strategy (for $M$ memory algorithms): Divide the set of $N$ experts into two groups which remember two different set of $M$ facts. Irrespective of what algorithm chooses to remember, we can always choose a set of questions to evaluate it on, such that algorithm makes $\approx M/2$ mistakes and half of the experts make $0$ mistake. Repeating this $\log(N)$ times (recursively on the set of experts which make $0$ mistakes) gives the required bound.

\begin{proof}[Proof of \Cref{thm:lowerbound}]
	We will divide our sequence of question answer pairs into two parts. We first discuss the first part where any algorithm will make at least $\floor{\log_{2c}N } \floor{M/2}$ mistakes and there exists an expert which makes $0$ mistakes. In the second part,  any algorithm will make at $\error$ mistakes and the best expert will make at most $\error$ mistakes.
	
	\paragraph{First part:} We will consider $\floor{\log_{2c} N}$ collections $C_k = \{(q_{k,j}, \Phi(q_{k,j}))\}_{j=1}^{2cM}$. Note that there are $2cM$ unique question answer pairs in  each collection. We will consider the following sequence: \[
		C_1, \Qcal_1, C_2, \Qcal_2, \ldots, C_{\floor{\log_{2c}N}}, \Qcal_{\floor{\log_{2c}N}}
	\] where each $\Qcal_i$ is a set of $M$ questions chosen by adversary that we will define later. We will show that for any algorithm $A$, there exists a choice of $\Qcal_i$'s such that the algorithm $A$ will make at least $\floor{M/2}$ mistakes on each $\Qcal_i$ and there exists a common expert which will make $0$ mistakes. Since, we do this for $i = 1$ to $\floor{\log_{2c}N}$, we get the desired result. 

	Next, we define our experts $\Ecal$ which we further divide into $2c$ groups $\Ecal_1, \ldots, \Ecal_{2c}$ each containing $\floor{N/2c}$ experts (we throw out the extra experts). Then, each $\Ecal_i$ is further divided into $2c$ groups $\Ecal_{i,1}, \ldots, \Ecal_{i, 2c}$. This tree like process is repeated $\floor{\log_{2c}N}$ times. Since, all the collections have distinct question answers and each expert (which is not thrown out) is in at least one of the leaf groups, we define the value function for expert $w$ in leaf group $\Ecal_{i_1, i_2,\ldots, i_{\floor{\log_{2c}N}}}$ for collection $C_k$ as \[
		v_e(q_{k,j}) = \begin{cases}
			k & \text{if}~ M(i_k-1) + 1 \leq j \leq Mi_k\\
			0 & \text{otherwise}
		\end{cases}
	\] Essentially, the experts in $\Ecal_{i_1, i_2,\ldots, i_{\floor{\log_{2c}N}}}$ when presented with collection $\{C_1, \ldots, C_k\}$ remembers only $M$ question answers from $C_k$, in particular $q_{k,j}$ for $j \in [M(i_k - 1), Mi_k]$.

	Since, the algorithm has only $cM$ memory, and there are in total $2cM$ question answers in each $C_k$, by pigeonhole principle there exists an $i_k \in [2c]$ such that $A$ remembers less than $\floor{M/2}$ question answers from $\{(q_{k,j}, \Phi(q_{k,j}))\}_{j= M (i_k - 1)}^{M i_k}$. Therefore, if we set $\Qcal_k = \{(q_{k,j}, \Phi(q_{k,j}))\}_{j= M (i_k - 1)}^{M i_k}$, we get that the algorithm $A$ makes at least $\floor{\log_{2c}N } \floor{M/2}$ mistakes. Also, by our setup, the experts in leaf group $\Ecal_{i_1, \ldots, i_{\floor{\log_{2c}N}}}$ will make $0$ mistakes. This proves the claim.

	\paragraph{Second part:} We consider the following sequence:
	\[
		(q_1,\Phi(q_1)), \ldots, (q_{cM+1},\Phi(q_{cM + 1})), q
	\] where unique question answer pairs $(q_i, \Phi(q_i))$ are shown  and then adversary chooses question $q$ to evaluate the experts and algorithm $A$. Note that by pigeonhole principle, there exists a question answer pair $(q_i, \Phi(q_i))$ such that the algorithm has not stored the answer for question $q_i$. Choosing $q = q_i$, the algorithm will make $1$ mistake and any expert will make at most $1$ mistake. Repeating this $\error$ times proves our claim.
\end{proof}
\newpage
\bibliographystyle{plainnat}
\bibliography{main}

\begin{thebibliography}{22}
\providecommand{\natexlab}[1]{#1}
\providecommand{\url}[1]{\texttt{#1}}
\expandafter\ifx\csname urlstyle\endcsname\relax
  \providecommand{\doi}[1]{doi: #1}\else
  \providecommand{\doi}{doi: \begingroup \urlstyle{rm}\Url}\fi

\bibitem[Anandkumar et~al.(2012)Anandkumar, Hsu, and
  Kakade]{anandkumar2012method}
Animashree Anandkumar, Daniel Hsu, and Sham~M. Kakade.
\newblock A method of moments for mixture models and hidden markov models.
\newblock In \emph{Conference on Learning Theory (COLT)}, 2012.

\bibitem[Bahdanau et~al.(2015)Bahdanau, Cho, and Bengio]{bahdanau2016neural}
Dzmitry Bahdanau, Kyunghyun Cho, and Yoshua Bengio.
\newblock Neural machine translation by jointly learning to align and
  translate.
\newblock In \emph{International Conference on Learning Representations
  (ICLR)}, 2015.

\bibitem[Cesa-Bianchi and Lugosi(2006)]{cesa-bianchi_lugosi_2006}
Nicolo Cesa-Bianchi and Gabor Lugosi.
\newblock \emph{Prediction, Learning, and Games}.
\newblock Cambridge University Press, 2006.

\bibitem[Charikar et~al.(2004)Charikar, Chen, and Farach-Colton]{countmin}
Moses Charikar, Kevin~C. Chen, and Martin Farach-Colton.
\newblock Finding frequent items in data streams.
\newblock \emph{Theor. Comput. Sci.}, 2004.

\bibitem[Dagan et~al.(2019)Dagan, Kur, and Shamir]{shamir19}
Yuval Dagan, Gil Kur, and Ohad Shamir.
\newblock Space lower bounds for linear prediction in the streaming model.
\newblock In \emph{Conference on Learning Theory (COLT)}, 2019.

\bibitem[Elman(1990)]{Elman90findingstructure}
Jeffrey~L. Elman.
\newblock Finding structure in time.
\newblock \emph{Cognitive Science}, 1990.

\bibitem[Garg et~al.(2018)Garg, Raz, and Tal]{garg18}
Sumegha Garg, Ran Raz, and Avishay Tal.
\newblock Extractor-based time-space lower bounds for learning.
\newblock In \emph{Symposium on Theory of Computing (STOC)}, 2018.

\bibitem[Gonen et~al.(2020)Gonen, Lovett, and Moshkovitz]{alon2020}
Alon Gonen, Shachar Lovett, and Michal Moshkovitz.
\newblock Towards a combinatorial characterization of bounded memory learning.
\newblock In \emph{Conference on Neural Information Processing Systems
  (NeurIPS)}, 2020.

\bibitem[Graves et~al.(2014)Graves, Wayne, and Danihelka]{graves2014neural}
Alex Graves, Greg Wayne, and Ivo Danihelka.
\newblock Neural turing machines.
\newblock \emph{arXiv preprint}, 2014.

\bibitem[Greenwald and Khanna(2001)]{khanna}
Michael Greenwald and Sanjeev Khanna.
\newblock Space-efficient online computation of quantile summaries.
\newblock \emph{International Conference on Management of Data (SIGMOD)}, 2001.

\bibitem[Hochreiter and Schmidhuber(1997)]{hochreiter1997long}
Sepp Hochreiter and J{\"u}rgen Schmidhuber.
\newblock Long short-term memory.
\newblock \emph{Neural computation}, 1997.

\bibitem[Hsu et~al.(2012)Hsu, Kakade, and Zhang]{hsu2012spectral}
Daniel Hsu, Sham~M. Kakade, and Tong Zhang.
\newblock A spectral algorithm for learning hidden markov models.
\newblock \emph{Journal of Computer and System Sciences (JCSS)}, 2012.

\bibitem[Littlestone and Warmuth(1994)]{mwu2012}
N.~Littlestone and M.K. Warmuth.
\newblock The weighted majority algorithm.
\newblock \emph{Information and Computation}, 1994.

\bibitem[Lu and Lu(2011)]{fu11}
Chi-Jen Lu and Wei-Fu Lu.
\newblock Making online decisions with bounded memory.
\newblock In \emph{Algorithmic Learning Theory (ALT)}, 2011.

\bibitem[Mitchell et~al.(2018)Mitchell, Cohen, Hruschka, Talukdar, Yang,
  Betteridge, Carlson, Dalvi, Gardner, Kisiel, et~al.]{mitchell2018never}
Tom Mitchell, William Cohen, Estevam Hruschka, Partha Talukdar, Bishan Yang,
  Justin Betteridge, Andrew Carlson, Bhanava Dalvi, Matt Gardner, Bryan Kisiel,
  et~al.
\newblock Never-ending learning.
\newblock \emph{Communications of the ACM}, 2018.

\bibitem[Morris(1978)]{morris}
Robert Morris.
\newblock Counting large numbers of events in small registers.
\newblock \emph{Commun. ACM}, 1978.

\bibitem[Raz(2016)]{ran16}
Ran Raz.
\newblock Fast learning requires good memory: A time-space lower bound for
  parity learning.
\newblock In \emph{Symposium on Theory of Computing (STOC)}, 2016.

\bibitem[Sharan et~al.(2018)Sharan, Kakade, Liang, and
  Valiant]{sharan2018prediction}
Vatsal Sharan, Sham Kakade, Percy Liang, and Gregory Valiant.
\newblock Prediction with a short memory.
\newblock In \emph{Symposium on Theory of Computing (STOC)}, 2018.

\bibitem[Steinhardt et~al.(2016)Steinhardt, Valiant, and Wager.]{jacob16}
Jacob Steinhardt, Gregory Valiant, and Stefan Wager.
\newblock Memory, communication, and statistical queries.
\newblock In \emph{Conference on Learning Theory (COLT)}, 2016.

\bibitem[Sukhbaatar et~al.(2021)Sukhbaatar, Ju, Poff, Roller, Szlam, Weston,
  and Fan]{forget2021}
Sainbayar Sukhbaatar, Da~Ju, Spencer Poff, Stephen Roller, Arthur Szlam, Jason
  Weston, and Angela Fan.
\newblock Not all memories are created equal: Learning to forget by expiring.
\newblock In \emph{International Conference on Machine Learning (ICML)}, 2021.

\bibitem[Thrun and Mitchell(1995)]{thrun1995lifelong}
Sebastian Thrun and Tom~M Mitchell.
\newblock Lifelong robot learning.
\newblock \emph{Robotics and autonomous systems}, 1995.

\bibitem[Weston et~al.(2015)Weston, Chopra, and Bordes]{weston2015memory}
Jason Weston, Sumit Chopra, and Antoine Bordes.
\newblock Memory networks.
\newblock In \emph{International Conference on Learning Representations
  (ICLR)}, 2015.

\end{thebibliography}
\end{document}